\documentclass[11pt]{article}
\usepackage[utf8]{inputenc}
\usepackage{amsmath,amssymb,amsthm}
\usepackage{graphicx}
\usepackage{booktabs}
\usepackage[margin=1in]{geometry}
\usepackage{colortbl}
\usepackage{algorithm}
\usepackage{algorithmic}
\usepackage{hyperref}
\usepackage{cite}
\usepackage{xcolor}
\usepackage{tikz}
\usepackage{pgfplots}
\pgfplotsset{compat=1.18}
\usetikzlibrary{arrows.meta, positioning, calc, fit, backgrounds, shadows, patterns}

\definecolor{deepBlue}{RGB}{41, 98, 155}
\definecolor{softRed}{RGB}{235, 87, 87}
\definecolor{mintGreen}{RGB}{46, 204, 113}
\definecolor{warmOrange}{RGB}{230, 126, 34}
\definecolor{orchidPurple}{RGB}{142, 68, 173}
\definecolor{logitblue}{RGB}{31,119,180}
\definecolor{xgborange}{RGB}{255,127,14}
\definecolor{rfgreen}{RGB}{44,160,44}
\definecolor{gbdtpurple}{RGB}{148,103,189}

\newtheorem{proposition}{Proposition}[section]
\newtheorem{corollary}{Corollary}[section]

\newcommand{\iconpicture}[2]{%
  \tikz[baseline={(base.base)}, transform shape=false, line cap=round, line join=round]{
    \path[use as bounding box] (-0.5,-0.5) rectangle (0.5,0.5);
    \node[inner sep=0pt] (base) at (0,0) {};
    #2
  }%
}

\newcommand{\websiteicon}{\iconpicture{deepBlue}{%
  \draw[deepBlue, line width=0.9pt, fill=deepBlue!6] (0,0) circle (0.36);
  \draw[deepBlue, line width=0.45pt] (0,0) ellipse (0.36 and 0.14);
  \draw[deepBlue, line width=0.35pt] (0,0.10) ellipse (0.30 and 0.11);
  \draw[deepBlue, line width=0.35pt] (0,-0.10) ellipse (0.30 and 0.11);
  \draw[deepBlue, line width=0.45pt] (0,0) ellipse (0.14 and 0.36);
  \draw[deepBlue, line width=0.35pt] (0.10,0) ellipse (0.11 and 0.30);
  \draw[deepBlue, line width=0.35pt] (-0.10,0) ellipse (0.11 and 0.30);
  \draw[deepBlue, line width=0.6pt] (-0.20,-0.40) -- (0.20,-0.40);
}}

\newcommand{\smallmoneybag}{%
  \tikz[baseline={(b.base)}, scale=0.95, transform shape=false, line cap=round, line join=round]{
    \path[use as bounding box] (-0.42,-0.45) rectangle (0.42,0.45);
    \node (b) at (0,0) {};
    \draw[warmOrange!85!black, line width=0.8pt, fill=warmOrange!12] (-0.10,0.10) ellipse (0.20 and 0.08);
    \draw[warmOrange!85!black, line width=0.8pt] (-0.30,0.10) -- (-0.30,-0.08);
    \draw[warmOrange!85!black, line width=0.8pt] (0.10,0.10) -- (0.10,-0.08);
    \draw[warmOrange!85!black, line width=0.8pt, fill=warmOrange!18] (-0.10,-0.08) ellipse (0.20 and 0.08);
    \draw[warmOrange!85!black, line width=0.9pt, fill=warmOrange!18] (0.12,-0.05) ellipse (0.22 and 0.09);
    \draw[warmOrange!85!black, line width=0.9pt] (-0.10,-0.05) -- (-0.10,-0.28);
    \draw[warmOrange!85!black, line width=0.9pt] (0.34,-0.05) -- (0.34,-0.28);
    \draw[warmOrange!85!black, line width=0.9pt, fill=warmOrange!28] (0.12,-0.28) ellipse (0.22 and 0.09);
    \node[warmOrange!85!black, font=\bfseries\tiny] at (0.12,-0.16) {\$};
  }%
}

\newcommand{\classifiericon}{\iconpicture{deepBlue}{%
  \foreach \i in {-0.2,0,0.2} \foreach \j in {-0.2,0,0.2} \fill[deepBlue!40] (\i,\j) circle (0.05);
  \draw[deepBlue, line width=1.1pt] (0.1,0.1) circle (0.25);
  \draw[deepBlue, line width=1.5pt] (0.28,-0.08) -- (0.45,-0.25);
}}

\newcommand{\vectoricon}{\iconpicture{deepBlue}{%
  \foreach \y in {0.18, 0, -0.18} \fill[deepBlue] (-0.38,\y) rectangle (0.38,\y+0.10);
}}

\newcommand{\shuffleicon}{\iconpicture{softRed}{%
  \draw[softRed, line width=0.75pt, ->] (-0.38,0.20) .. controls (0,0.20) and (0,-0.20) .. (0.38,-0.20);
  \draw[softRed, line width=0.75pt, ->] (-0.38,-0.20) .. controls (0,-0.20) and (0,0.20) .. (0.38,0.20);
}}

\newcommand{\attackericon}{\iconpicture{softRed}{%
  \draw[softRed, line width=0.75pt, fill=softRed!10] (0,0.25) circle (0.20);
  \draw[softRed, line width=0.75pt, fill=softRed] (-0.28,0.18) rectangle (0.28,0.30);
  \draw[softRed, line width=0.75pt] (0,0.08) -- (0,-0.30);
  \draw[softRed, line width=0.75pt] (-0.28,-0.05) -- (0.28,-0.05);
}}

\title{\LARGE Robustness, Cost, and Attack-Surface Concentration\\in Phishing Detection}

\author{
Julian Allagan$^{1,*}$, Mohamed Elbakary$^{1}$, Zohreh Safari$^{1}$, \\  Weizheng Gao$^{1}$, Gabrielle Morgan$^{1}$, Essence Morgan$^{1}$ \\
Vladimir Deriglazov$^{1}$\\[0.5em]
\small $^{1}$Department of Mathematics, Computer Science, and Engineering Technology\\
\small Elizabeth City State University, Elizabeth City, NC 27909, USA\\[0.3em]
\small $^{*}$Corresponding author: \texttt{adallagan@ecsu.edu}
}

\date{}

\begin{document}

\maketitle

\begin{abstract}
Phishing detectors built on engineered website features attain near-perfect accuracy under i.i.d.\ evaluation, yet deployment security depends on robustness to post-deployment feature manipulation. We study this gap through a cost-aware evasion framework that models discrete, monotone feature edits under explicit attacker budgets. Three diagnostics are introduced: minimal evasion cost (MEC), the evasion survival rate $S(B)$, and the robustness concentration index (RCI).

On the UCI Phishing Websites benchmark (11\,055 instances, 30 ternary features), Logistic Regression, Random Forests, Gradient Boosted Trees, and XGBoost all achieve $\mathrm{AUC}\ge 0.979$ under static evaluation. Under budgeted sanitization-style evasion, robustness converges across architectures: the median MEC equals 2 with full features, and over 80\% of successful minimal-cost evasions concentrate on three low-cost surface features. Feature restriction improves robustness only when it removes all dominant low-cost transitions. Under strict cost schedules, infrastructure-leaning feature sets exhibit 17-19\% infeasible mass for ensemble models, while the median MEC among evadable instances remains unchanged. We formalize this convergence: if a positive fraction of correctly detected phishing instances admit evasion through a single feature transition of minimal cost $c_{\min}$, no classifier can raise the corresponding MEC quantile above $c_{\min}$ without modifying the feature representation or cost model. Adversarial robustness in phishing detection is governed by feature economics rather than model complexity.
\end{abstract}

\section{Introduction}
\label{sec:introduction}

Phishing detection is inherently adversarial. Attackers adapt observable website characteristics to evade classification, while defenders evaluate models under static train--test splits. A classifier may achieve near-perfect held-out accuracy yet remain operationally fragile when its predictions rest on surface-level attributes alterable at low cost. Recent studies report detection accuracies exceeding 95\% using refined feature engineering and ensemble methods \cite{basit2021comprehensive,mohammad2014predicting,sahingoz2019machine,do2022deep}, but these results assume a passive threat model in which adversarial adaptation is excluded. In deployed systems this assumption rarely holds \cite{biggio2018wild,apruzzese2023real,biggio2013evasion}.

The tension is acute because manipulation costs are asymmetric. Presentation-layer cues---URL structure, HTML artifacts, certificate presentation---are inexpensive to modify, whereas infrastructure-coupled signals such as domain age, DNS records, and traffic require sustained investment or third-party validation \cite{khonji2013phishing,lin2021phishpedia,oest2020sunrise,bijmans2021catching}. Many robustness analyses rely on continuous perturbation models that abstract away discrete feature semantics, or adopt worst-case threat models that ignore economically plausible attacker behavior \cite{zhou2022adversarial,pierazzi2020intriguing,carlini2017towards}.

Phishing robustness studies fall broadly into three categories. Continuous perturbation approaches apply $\ell_p$-bounded adversarial examples to feature vectors, treating each coordinate as a real-valued input amenable to gradient-based attack \cite{zhou2022adversarial,goodfellow2015explaining}. While methodologically convenient, this abstraction obscures the discrete, semantically constrained nature of website feature edits. Heuristic attack studies evaluate classifiers against hand-crafted manipulation strategies without formalizing attacker cost or optimality \cite{basit2021comprehensive,sahingoz2019machine,corona2017deltaphish}. Problem-space constraint work emphasizes that adversarial perturbations must satisfy domain constraints, yet typically does not assign explicit economic costs to individual transitions or characterize the resulting attack-surface structure \cite{pierazzi2020intriguing,apruzzese2023real,xu2016automatically}.

We adopt a complementary perspective. Evasion is formulated as exact shortest-path search over a cost-weighted discrete transition graph. We introduce concentration diagnostics---RCI and FirstTop1---as primary structural indicators, and establish an architecture-independent cost-floor bound. The combination of exact discrete-cost optimization, attack-surface concentration measurement, and a formal robustness ceiling distinguishes this work from prior analyses.

To bridge the gap between static evaluation and adversarial deployment, we develop a cost-aware adversarial evaluation framework that assigns explicit costs to discrete feature edits and evaluates classifiers under bounded attacker budgets. We study sanitization-style evasion under monotone edits, where the attacker removes phishing indicators and pushes feature values toward legitimate states. This threat model represents a lower bound on attacker capability: it excludes anti-feature injection, extractor-level attacks, and non-monotone manipulation, all of which can only expand the feasible action set and reduce MEC. The restriction to monotone edits is operationally motivated by empirical evidence that most phishing campaigns are short-lived (24--72 hours), favoring indicator removal over infrastructure construction \cite{oest2020sunrise,bijmans2021catching}. Section~\ref{sec:discussion_conclusion} discusses how relaxing this restriction would affect concentration and the cost floor.

Rather than measuring aggregate degradation, we address a structural question central to defensive design: under budget-constrained manipulation, do evasion pathways disperse across many features or collapse onto a small attack surface? The answer determines whether architectural complexity redistributes adversarial risk or leaves dominant failure modes intact.

We operationalize this analysis with three diagnostics. Minimal evasion cost (MEC) is the smallest cumulative cost required to induce misclassification for a correctly detected phishing instance. The evasion survival rate $S(B)=\Pr(\mathrm{MEC}(x)>B)$ measures resistance at attacker budget $B$. The robustness concentration index (RCI) quantifies whether successful minimal-cost edits are diffuse or concentrated on a small subset of features. Empirically, across models and full feature sets, evasion succeeds under modest budgets (median MEC $= 2$), and more than 80\% of traces concentrate on three low-cost surface features. We formalize this convergence with a structural result: when a nontrivial fraction of instances admit evasion via a single feature transition of minimal cost, no classifier architecture can raise the corresponding robustness quantiles without modifying the feature space or cost model. We term this phenomenon action-set-limited invariance.

\section{Methods}
\label{sec:methods}

We model post-deployment evasion as a shortest-path problem on a directed graph whose nodes are discrete feature vectors, whose edges represent admissible monotone manipulations, and whose edge weights encode attacker cost.

\subsection{Threat model}
Let $f:\mathcal{X}\to\{-1,+1\}$ denote a deployed classifier ($-1$ for phishing, $+1$ for legitimate). Given a phishing instance $x$ correctly classified as malicious, the attacker seeks $x'$ with $f(x')=+1$ subject to a finite manipulation budget $B$. Feature vectors lie in $\{-1,0,+1\}^d$, encoding phishing-indicative, neutral, and legitimate states. Edits are monotone: a transition $v\to v'$ is admissible only if $v'\ge v$ under the ordering $-1<0<+1$. Reverse transitions incur infinite cost. This models sanitization-style evasion in which attackers remove suspicious indicators rather than inject adversarial anti-features. The attacker possesses feature-level knowledge---awareness that detection relies on discrete website features and coarse understanding of surface versus infrastructure cost asymmetries, consistent with publicly documented detection pipelines \cite{apruzzese2023real,das2020sok}---but has no access to model parameters, training data, confidence scores, or gradients.

This threat model constitutes a lower bound on attacker capability. Non-monotone edits (injecting benign artifacts), extractor-level manipulation (exploiting parser ambiguities to alter computed features without semantic change), and anti-feature attacks all enlarge the feasible action set. Any such enlargement can only decrease MEC and potentially increase concentration. The cost floor established under monotone edits therefore provides a conservative bound: if robustness is fragile under sanitization-only attackers, it is at least as fragile under more capable adversaries.

We compute MEC values via uniform-cost search (Algorithm~\ref{alg:mec}), yielding exact shortest paths within the prescribed budget. Exact MEC represents an upper bound on evasion efficiency under the defined action set: query-limited attackers may fail to discover optimal evasions, raising empirical survival rates, but the structural cost floor persists whenever low-cost transitions remain available.

\subsection{Cost schedules}
Each admissible edit $(j,v\to v')$ incurs nonnegative cost $c(j,v\to v')$. For an instance $x$ transformed to $x'$, the cumulative cost is additive:
\begin{equation}\nonumber
C(x\to x')=\sum_{(j,v\to v')\in\Delta(x,x')} c(j,v\to v').
\end{equation}
The feasible action set at budget $B$ is $\mathcal{A}_B(x)=\{x'\in\mathcal{X} : C(x\to x') \le B\}$.

Costs represent dimensionless operational friction---the difficulty of effecting a manipulation within a phishing campaign's operational window---rather than direct monetary expenditure. We calibrate using a time-to-effect principle: one cost unit corresponds to a manipulation executable within a single day by the campaign operator; four units correspond to changes requiring multi-week external accumulation (DNS propagation, organic traffic growth, reputation accrual). This calibration reflects documented campaign lifecycles: Oest et al.\ \cite{oest2020sunrise} report median campaign durations under 21 hours, with 95\% retired within 72 hours, establishing surface-level edits as effectively ``free'' within the operational window and infrastructure changes as largely infeasible. Bijmans et al.\ \cite{bijmans2021catching} corroborate this timeline for phishing-kit deployments.

What matters for the structural conclusions is the cost ordering---surface features are strictly cheaper than infrastructure features---rather than exact magnitudes. Section~\ref{sec:sensitivity} demonstrates that proportional cost scaling shifts the median MEC linearly while preserving feature ordering, concentration structure, and architecture invariance.

We consider two schedules. The base schedule assigns low cost to surface features (URL structure, HTML presentation; $c=1,2$) and higher cost to infrastructure features (domain age, DNS, traffic, reputation; $c=4,8$). The strict schedule coincides with the base except that infrastructure-feature upgrades to the fully legitimate state are disallowed ($c=\infty$), modeling horizons in which complete infrastructure legitimization is infeasible. Table~\ref{tab:cost_model_full} summarizes both schedules. All experiments use $B_{\max}=18$.

\begin{table}[t]
\centering
\caption{Feature manipulation cost schedules with operational time horizons. Costs are calibrated by the time-to-effect principle: 1 unit $\approx$ changes feasible within a day; 4 units $\approx$ multi-week external accumulation.}
\label{tab:cost_model_full}
\footnotesize
\setlength{\tabcolsep}{3pt}
\begin{tabular}{@{}llccccccl@{}}
\toprule
& & \multicolumn{3}{c}{Base Schedule} & \multicolumn{3}{c}{Strict Schedule} & \\
\cmidrule(lr){3-5} \cmidrule(lr){6-8}
Feature Group & Examples & $-1{\to}0$ & $-1{\to}1$ & $0{\to}1$ & $-1{\to}0$ & $-1{\to}1$ & $0{\to}1$ & Time Horizon \\
\midrule
Surface & URL\_of\_Anchor, SFH, & 1 & 2 & 1 & 1 & 2 & 1 & hours to days \\
& Prefix\_Suffix, SSLfinal\_State & & & & & & \\
Semi-domain & Domain\_Reg\_Length, & 3 & 6 & 3 & 3 & 6 & 3 & days to weeks \\
& Google\_Index & & & & & & \\
Infrastructure & web\_traffic, DNSRecord, & 4 & 8 & 4 & 4 & $\infty$ & $\infty$ & weeks to months \\
& age\_of\_domain, Page\_Rank & & & & & & \\
\bottomrule
\end{tabular}
\end{table}

SSLfinal\_State is classified as surface-level because certificate presentation can be modified through front-end configuration (e.g., deploying a free DV certificate via Let's Encrypt), without sustained infrastructure investment. Under the time-to-effect principle, this operation falls within the single-day horizon. Reclassifying SSLfinal\_State as semi-domain is examined in the sensitivity analysis.

\subsection{Dataset, models, and conditioning}
We use the UCI Phishing Websites benchmark \cite{mohammad2015uci}: 11\,055 instances described by 30 ternary features in $\{-1,0,+1\}$, with 4\,898 phishing and 6\,157 legitimate websites. A stratified 75/25 train--test split (seed 1337) yields 2\,764 test instances including 1\,225 phishing samples. Four classifier families are evaluated: Logistic Regression ($\ell_2$ regularization, $C=1.0$), Random Forests (100 trees, max depth 10), Gradient Boosted Decision Trees (100 estimators, learning rate 0.1, max depth 6), and XGBoost with matched hyperparameters. Classification uses a fixed threshold of 0.5 on each model's native \texttt{predict\_proba} output, held constant across all models and configurations. Implementations use scikit-learn 1.0.2 and xgboost 1.5.0.

Robustness is evaluated on the conditioning set $\mathcal{P}_0=\{x\in\mathcal{P}_{\mathrm{test}} : f(x)=-1\}$, the phishing instances correctly classified as malicious. Conditioning isolates post-detection evasion and separates robustness from baseline classification error. For cross-model comparisons, we intersect $\mathcal{P}_0$ across all four models and uniformly sample $n_{\mathrm{eval}}=300$ instances. Table~\ref{tab:p0_validation} verifies that this intersection does not bias conclusions by comparing metrics on each model's full $\mathcal{P}_0$ with the intersection sample.

\subsection{Robustness metrics}
For each $x\in\mathcal{P}_0$, the minimal evasion cost is
\begin{equation}\nonumber
\mathrm{MEC}(x)=\inf\{C(x\to x') : f(x')=+1\},
\end{equation}
computed exactly via uniform-cost search (Algorithm~\ref{alg:mec}). The search is complete up to $B_{\max}$ and returns $\mathrm{MEC}(x)=\infty$ when no evasion exists within budget. Median runtime is 0.3\,s per instance for full feature sets.

Resistance at attacker budget $B$ is summarized by the evasion survival rate $S(B)=\Pr(\mathrm{MEC}(x)>B \mid x\in\mathcal{P}_0)$. Aggregate robustness is captured by the feature robustness index
\begin{equation}\nonumber
\mathrm{FRI}=\frac{1}{B_{\max}}\int_0^{B_{\max}} S(B)\,dB,
\end{equation}
approximated by a left Riemann sum over integer budgets. FRI measures the expected fraction of the budget range over which a randomly selected instance from $\mathcal{P}_0$ resists evasion; equivalently, it is the normalized area under the survival curve up to $B_{\max}$. FRI incorporates infeasible mass (instances with $\mathrm{MEC}=\infty$ contribute $S(B)=1$ for all $B$), while median and quartile MEC are computed over finite values only. This separation distinguishes overall resistance from the cost distribution among evadable instances.

To examine attack-surface structure, let $N_j$ denote the total number of edits applied to feature $j$ across all successful minimal-cost traces. The robustness concentration index is
\begin{equation}\nonumber
\mathrm{RCI}_k=\frac{\sum_{j\in\mathrm{Top}\text{-}k} N_j}{\sum_{j=1}^d N_j},
\end{equation}
measuring the fraction of adversarial effort concentrated on the $k$ most frequently edited features. When multiple optimal paths share identical cost, deterministic priority-queue tie-breaking selects a canonical trace. This affects the representative path used for concentration metrics but not MEC itself; recomputing $\mathrm{RCI}_3$ under 10 randomized tie-break orders yields standard deviation below 0.02 in all configurations.

To isolate first-step bottlenecks, let $j_1(x)$ denote the first-edited feature in the canonical minimal-cost trace. Define $j^\star=\arg\max_j |\{x : j_1(x)=j\}|$. The FirstTop1 index is
\begin{equation}\nonumber
\mathrm{FirstTop1}
=
\frac{|\{x : j_1(x)=j^\star\}|}
{|\{x : \mathrm{MEC}(x)<\infty\}|},
\end{equation}
capturing single-feature bottlenecks at the initial decision step.

Figure~\ref{fig:framework} summarizes the evaluation pipeline and robustness diagnostics.

\begin{figure}[ht]
\centering
\begin{tikzpicture}[
  scale=0.7, transform shape,
  font=\sffamily\small,
  >={Latex[length=3mm, width=2mm]},
  thick,
  mainbox/.style={rectangle, rounded corners=8pt, draw=#1, line width=1.2pt, fill=#1!5, minimum height=1.6cm, minimum width=3.0cm, align=center, drop shadow={shadow xshift=0.4mm, shadow yshift=-0.4mm, opacity=0.25}},
  threatbox/.style={rectangle, rounded corners=8pt, draw=softRed, line width=1.5pt, densely dashed, fill=softRed!2, minimum height=1.4cm, minimum width=2.8cm, align=center, drop shadow={shadow xshift=0.4mm, shadow yshift=-0.4mm, opacity=0.20}},
  metricbox/.style={rectangle, rounded corners=6pt, draw=orchidPurple!70, line width=1pt, fill=orchidPurple!5, minimum height=1.7cm, minimum width=3.2cm, align=center, drop shadow={shadow xshift=0.35mm, shadow yshift=-0.35mm, opacity=0.18}}
]

\node[mainbox=deepBlue] (input) at (0,0) {\websiteicon\\[0.1em]Phishing Site\\{\scriptsize Instance $x$}};
\node[mainbox=deepBlue, right=1.2cm of input] (features) {\vectoricon\\[0.1em]Features\\{\scriptsize $x \in \{-1,0,1\}^d$}};
\node[mainbox=softRed, right=1.2cm of features] (edits) {\shuffleicon\\[0.1em]Manipulation\\{\scriptsize $x \xrightarrow{\mathcal{A}_B} x'$}};
\node[mainbox=deepBlue, right=1.2cm of edits] (model) {\classifiericon\\[0.1em]Classifier $f$\\{\scriptsize LR / RF / GBDT / XGB}};
\node[mainbox=softRed, right=1.2cm of model, yshift=0.8cm, minimum height=1.1cm] (phish) {Phishing\\{\scriptsize $f(x)=-1$}};
\node[mainbox=mintGreen, right=1.2cm of model, yshift=-0.8cm, minimum height=1.1cm] (legit) {Legitimate\\{\scriptsize $f(x')=+1$}};
\node[threatbox, above=0.7cm of edits] (attacker) {\attackericon\\[0.1em]Attacker\\{\scriptsize Budget $B \leq B_{\max}$}};
\node[mainbox=warmOrange, fill=warmOrange!10, below=0.7cm of edits, minimum height=1.3cm, minimum width=4.0cm, text width=4.2cm] (cost) {
  \begin{minipage}[c]{0.12\textwidth}\centering \smallmoneybag\end{minipage}%
  \begin{minipage}[c]{0.85\textwidth}\centering\vspace{2pt}Cost Model\\[-0.2em]{\scriptsize Surface: $c=1,2$ \quad Infra: $c=4,8$}\end{minipage}
};

\node[below=3.8cm of model] (diag_center) {};
\begin{scope}[on background layer]
  \node[rectangle, rounded corners=10pt, draw=orchidPurple!40, line width=1.5pt, densely dotted, fill=orchidPurple!2, minimum width=11.5cm, minimum height=2.6cm] (diag_box) at (diag_center) {};
\end{scope}
\node[font=\bfseries, text=orchidPurple, anchor=north] at ([yshift=0.6cm]diag_box.north) {Robustness Diagnostics};
\node[metricbox] (mec) at ([xshift=-3.8cm, yshift=-0.1cm]diag_center) {MEC\\{\scriptsize Min Evasion Cost}\\[0.15em]$\displaystyle\min_{\boldsymbol{x}'} C(\boldsymbol{x}\!\to\!\boldsymbol{x}')$};
\node[metricbox] (fri) at ([yshift=-0.1cm]diag_center) {FRI\\{\scriptsize Robustness Index}\\[0.15em]$\displaystyle\frac{1}{B_{\max}}\!\int_0^{B_{\max}}\!\! S(B)\,dB$};
\node[metricbox] (rci) at ([xshift=3.8cm, yshift=-0.1cm]diag_center) {RCI$_k$\\{\scriptsize Concentration}\\[0.15em]$\displaystyle\frac{\sum_{j \in \text{Top-}k} N_j}{\sum_{j=1}^{d} N_j}$};

\draw[->] (input) -- (features);
\draw[->] (features) -- (edits);
\draw[->] (edits) -- (model);
\draw[->] (model.east) -- ++(0.3,0) |- (phish.west);
\draw[->] (model.east) -- ++(0.3,0) |- (legit.west);
\draw[->, softRed, dashed] (attacker.south) -- (edits.north);
\draw[->, warmOrange] (cost.north) -- (edits.south);

\coordinate (eval_join) at ($(phish.east)!0.5!(legit.east) + (0.7,0)$);
\draw[orchidPurple, line width=1.3pt] (phish.east) -- ++(0.4,0) |- (eval_join);
\draw[orchidPurple, line width=1.3pt] (legit.east) -- ++(0.4,0) |- (eval_join);
\draw[->, orchidPurple, line width=1.3pt] (eval_join) -- (eval_join |- diag_box.center) -- (diag_box.east);

\end{tikzpicture}
\caption{Cost-aware adversarial robustness framework with MEC, survival curves, and attack-surface concentration.}
\label{fig:framework}
\end{figure}
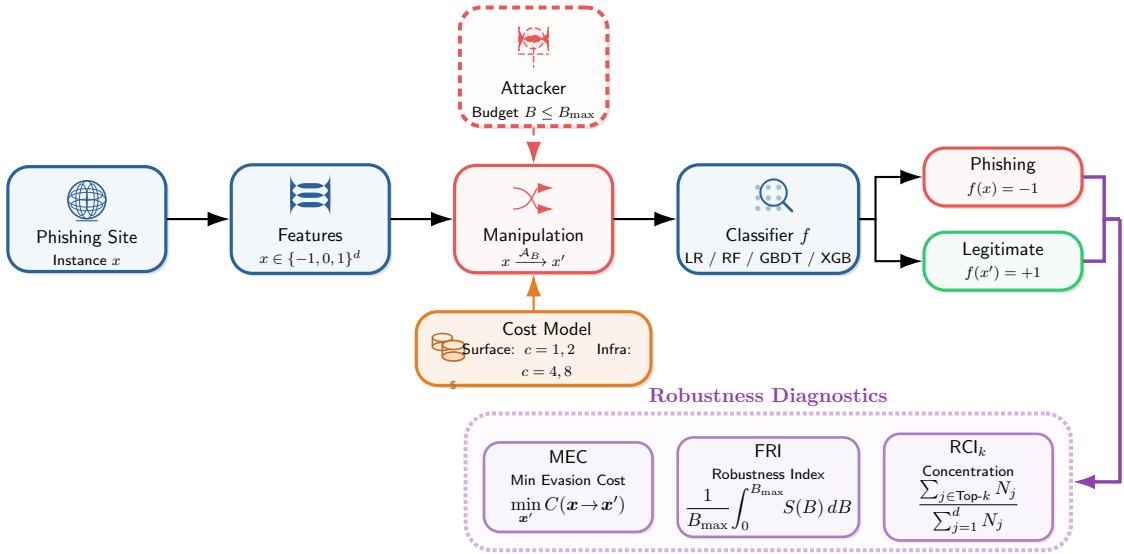

\begin{algorithm}[t]
\caption{Minimal Evasion Cost via Uniform-Cost Search}
\label{alg:mec}
\begin{algorithmic}[1]
\REQUIRE Classifier $f$, instance $x$, cost function $c$, budget $B_{\max}$
\ENSURE $\mathrm{MEC}(x)$ or $\infty$ if no evasion within budget
\IF{$f(x)=+1$}
    \STATE \RETURN $0$
\ENDIF
\STATE $Q\leftarrow\{(0,x)\}$, \quad $V\leftarrow\emptyset$
\WHILE{$Q\neq\emptyset$}
    \STATE $(c_{\mathrm{curr}},x_{\mathrm{curr}})\leftarrow Q.\mathrm{pop\_min}()$
    \IF{$x_{\mathrm{curr}}\in V$}
        \STATE \textbf{continue}
    \ENDIF
    \STATE $V\leftarrow V\cup\{x_{\mathrm{curr}}\}$
    \IF{$f(x_{\mathrm{curr}})=+1$}
        \STATE \RETURN $c_{\mathrm{curr}}$
    \ENDIF
    \FOR{each feature $j\in\{1,\ldots,d\}$}
        \FOR{each target value $v'>x_{\mathrm{curr}}[j]$ in $\{0,1\}$}
            \STATE $c_{\mathrm{edit}}\leftarrow c(j,x_{\mathrm{curr}}[j]\to v')$
            \IF{$c_{\mathrm{curr}}+c_{\mathrm{edit}}\le B_{\max}$}
                \STATE $x'\leftarrow x_{\mathrm{curr}}$ with $x'[j]\leftarrow v'$
                \STATE push $(c_{\mathrm{curr}}+c_{\mathrm{edit}},x')$ to $Q$
            \ENDIF
        \ENDFOR
    \ENDFOR
\ENDWHILE
\STATE \RETURN $\infty$
\end{algorithmic}
\end{algorithm}

\section{Feature Economics and Robustness Limits}
\label{sec:theory}

We now establish a structural limit imposed by feature-level manipulation costs. The result identifies a cost floor that bounds achievable robustness independently of model architecture.

\begin{proposition}[Cost floor]
\label{prop:cost_floor}
Let $c_{\min}=\min_{j,v,v'} c(j,v\to v')$ be the minimum cost among all admissible single-feature transitions. Fix a classifier $f$ and let $\mathcal{P}_0$ denote the set of phishing instances correctly detected by $f$. If a fraction $\alpha>0$ of instances in $\mathcal{P}_0$ admit evasion via a single transition of cost $c_{\min}$, then
\begin{equation}\nonumber
\Pr\!\left(\mathrm{MEC} \le c_{\min} \mid x\in\mathcal{P}_0\right)\ge \alpha.
\end{equation}
In particular, if $\alpha\ge \tfrac12$, then $\mathrm{median}(\mathrm{MEC}) \le c_{\min}$. Hence the $\alpha$-quantile of the MEC distribution cannot exceed $c_{\min}$ without modifying the feature space or cost schedule.
\end{proposition}

\begin{proof}
For each $x\in\mathcal{P}_0$ admitting a single-feature evasion $(j,v\to v')$ of cost $c_{\min}$, one has $\mathrm{MEC}(x)\le c_{\min}$ by definition of the infimum. Since these instances constitute at least an $\alpha$ fraction of $\mathcal{P}_0$, the distributional bound follows directly. The median statement is an immediate consequence of the definition: when $\alpha\ge\tfrac12$, at least half the probability mass lies at or below $c_{\min}$.
\end{proof}

The force of Proposition~\ref{prop:cost_floor} is not in the proof technique---which is elementary---but in the structural invariance it implies. Regardless of how a classifier partitions feature space, any instance that lies within a single cheap transition of a legitimate-classified region is evadable at cost $c_{\min}$. Whether that fraction $\alpha$ is large depends on the interaction between the cost landscape and the classifier's decision boundary, and the empirical contribution of this work is to show that $\alpha$ is indeed large across all tested architectures.

\begin{corollary}[Action-set-limited invariance]
\label{cor:invariance}
Fix $\mathcal{X}=\{-1,0,+1\}^d$ and a monotone cost function $c$ with minimum transition cost $c_{\min}$. Let $\{f_1,\ldots,f_K\}$ be classifiers evaluated on a common conditioning set $\mathcal{P}_0$, and suppose that for each $f_k$ at least an $\alpha>0$ fraction of $\mathcal{P}_0$ admits single-transition evasion at cost $c_{\min}$. Then for every $k$:
\begin{equation}\nonumber
\Pr\!\left(\mathrm{MEC}_{f_k}\le c_{\min} \mid x\in\mathcal{P}_0\right)\ge \alpha.
\end{equation}
Architectural variation alone cannot exceed this bound. Invariance breaks when the feature representation changes (removing or hardening features), when the cost schedule is modified (raising $c_{\min}$), or when the feature extractor is made robust to manipulation (reducing the attacker's effective action set).
\end{corollary}

\begin{proof}
Under a common action set and shared conditioning set, Proposition~\ref{prop:cost_floor} applies identically to each $f_k$.
\end{proof}

Table~\ref{tab:alpha_cost_floor} reports the empirical fraction $\hat{\alpha}_{c_{\min}} = \Pr(\mathrm{MEC}\le c_{\min}\mid x\in\mathcal{P}_0)$ for each configuration, confirming that the cost floor binds in practice.

\begin{table}[t]
\centering
\caption{Empirical mass at the cost floor across feature configurations.}
\label{tab:alpha_cost_floor}
\small
\setlength{\tabcolsep}{5pt}
\begin{tabular}{llccc}
\toprule
Feature set & Schedule & $c_{\min}$ & $\hat{\alpha}_{c_{\min}}$ & Observed median MEC \\
\midrule
Full  & base   & 1 & 0.31 & 2 \\
Full  & strict & 1 & 0.33 & 2 \\
RA-8  & base   & 1 & 0.28 & 2 \\
RA-8  & strict & 1 & 0.15 & 2 \\
VA-7b & base   & 1 & 0.62 & 1 \\
\bottomrule
\end{tabular}
\end{table}

We evaluate six feature configurations. The full set contains all 30 features. AAS-12a ($d=12$) and AAS-11b ($d=11$) retain high-information-gain features. RA-8 ($d=8$) emphasizes infrastructure-leaning signals but includes SSLfinal\_State. VA-8a ($d=8$) and VA-7b ($d=7$) contain only presentation-layer features. In all cases $c_{\min}\in\{1,2\}$, so by Proposition~\ref{prop:cost_floor}, median MEC cannot exceed these values unless all transitions at that cost are removed.

\section{Results}
\label{sec:results}

Table~\ref{tab:iid} reports held-out classification performance on the full feature set. All models achieve strong discrimination (AUC between 0.979 and 0.995), suggesting reliable deployment under static evaluation. The adversarial analysis below demonstrates that this conclusion does not survive once feature manipulation is permitted.

\begin{table}[t]
\centering
\caption{Held-out classification performance (Full feature set, threshold $=0.5$).}
\label{tab:iid}
\begin{tabular}{lccc}
\toprule
Model & Accuracy & AUC & Phishing TPR \\
\midrule
Logistic Regression & 0.927 & 0.979 & 0.900 \\
Random Forest & 0.950 & 0.993 & 0.913 \\
GBDT & 0.953 & 0.990 & 0.932 \\
XGBoost & 0.965 & 0.995 & 0.953 \\
\bottomrule
\end{tabular}
\end{table}

\begin{table}[t]
\centering
\caption{Robustness metrics on full $\mathcal{P}_0$ versus the 300-instance cross-model intersection (Full/base).}
\label{tab:p0_validation}
\small
\begin{tabular}{lccccc}
\toprule
& \multicolumn{2}{c}{Median MEC} & & \multicolumn{2}{c}{RCI$_3$} \\
\cmidrule{2-3} \cmidrule{5-6}
Model & Full $\mathcal{P}_0$ & Intersection & & Full $\mathcal{P}_0$ & Intersection \\
\midrule
Logistic Regression & 2 & 2 & & 0.96 & 0.96 \\
Random Forest & 2 & 2 & & 0.85 & 0.84 \\
GBDT & 2 & 2 & & 0.89 & 0.89 \\
XGBoost & 2 & 2 & & 0.82 & 0.82 \\
\bottomrule
\end{tabular}
\end{table}

Table~\ref{tab:main} presents the central robustness results. Two regularities dominate across all configurations.

First, robustness is bounded by a low effective cost floor. On the full feature set, all architectures exhibit median MEC $= 2$ with narrow interquartile ranges and small FRI values. Although single-feature transitions of cost 1 exist, the empirical mass at cost 1 falls below one half, so the median binds at the next effective threshold. The convergence of linear, bagging, and boosting models to the same median MEC confirms the action-set-limited invariance of Corollary~\ref{cor:invariance}.

Second, successful evasion concentrates sharply on a small feature subset. For the full feature set under the base schedule, $\mathrm{RCI}_3$ exceeds 0.78 across models and reaches 0.96 for logistic regression. Evasion traces collapse onto low-cost, high-influence features rather than dispersing across the representation. The 95\% bootstrap confidence intervals (200 resamples) confirm that these patterns are statistically stable: median MEC $=[2,2]$ for all models, $\mathrm{RCI}_3$ within $\pm 0.03$, and FRI within $\pm 0.01$ (Table~\ref{tab:bootstrap_ci}).

The RA-8 configuration makes the cost-floor mechanism explicit. Despite emphasizing infrastructure features, RA-8 retains SSLfinal\_State, a low-cost surface coordinate. Median MEC remains 2, while concentration becomes nearly degenerate ($\mathrm{RCI}_3 \approx 1$, $\mathrm{FirstTop1} \approx 1$). The surface-only VA-7b set exhibits the lowest robustness (median MEC $=1$, FRI $< 0.05$).

Cost schedules matter only when they eliminate dominant cheap paths. This occurs in RA-8 under the strict schedule: ensemble models exhibit 17--19\% infeasible mass, raising FRI to 0.23--0.25, while median MEC among evadable instances remains 2. The gain arises from blocked feasibility rather than uniformly higher evasion costs. Logistic regression remains fully evadable in RA-8/strict, indicating alternative low-cost paths in the linear boundary.

\begin{table}[t]
\centering
\caption{Robustness across feature sets and schedules. NoEvasion reports infeasible mass within $B_{\max}=18$.}
\label{tab:main}
\small
\setlength{\tabcolsep}{3.2pt}
\begin{tabular}{llcccccccc}
\toprule
Features & Sched. & Model & Acc & FRI & MEC & [Q1,Q3] & RCI$_3$ & FT1 & NoEv \\
\midrule
Full & base & Logit & .927 & .076 & 2 & [1,2] & .961 & .850 & 0\% \\
Full & base & RF & .950 & .092 & 2 & [2,3] & .843 & .580 & 0\% \\
Full & base & GBDT & .953 & .076 & 2 & [2,2] & .892 & .370 & 0\% \\
Full & base & XGB & .965 & .092 & 2 & [2,3] & .815 & .440 & 0\% \\
\midrule
Full & strict & Logit & .927 & .077 & 2 & [1,2] & .975 & .847 & 0\% \\
Full & strict & RF & .950 & .093 & 2 & [2,3] & .843 & .540 & 0\% \\
Full & strict & GBDT & .953 & .075 & 2 & [2,2] & .854 & .397 & 0\% \\
Full & strict & XGB & .965 & .091 & 2 & [2,3] & .784 & .413 & 0\% \\
\midrule
RA-8 & base & Logit & .869 & .081 & 2 & [1,2] & 1.00 & .993 & 0\% \\
RA-8 & base & RF & .900 & .104 & 2 & [2,2] & .972 & .973 & 0\% \\
RA-8 & base & GBDT & .899 & .091 & 2 & [2,2] & 1.00 & .993 & 0\% \\
RA-8 & base & XGB & .904 & .096 & 2 & [2,2] & .986 & .990 & 0\% \\
\midrule
RA-8 & strict & Logit & .869 & .086 & 2 & [1,2] & .980 & .997 & 0\% \\
RA-8 & strict & RF & .900 & .247 & 2 & [2,2] & 1.00 & 1.00 & 18\% \\
RA-8 & strict & GBDT & .899 & .231 & 2 & [1.75,2] & 1.00 & 1.00 & 17\% \\
RA-8 & strict & XGB & .904 & .251 & 2 & [2,2] & 1.00 & 1.00 & 19\% \\
\midrule
VA-7b & base & Logit & .862 & .049 & 1 & [1,2] & .983 & .897 & 0\% \\
VA-7b & base & RF & .871 & .042 & 1 & [1,2] & .987 & .827 & 0\% \\
VA-7b & base & GBDT & .869 & .046 & 1 & [1,2] & .997 & .827 & 0\% \\
VA-7b & base & XGB & .869 & .044 & 1 & [1,2] & .880 & .827 & 0\% \\
\bottomrule
\end{tabular}
\end{table}

\begin{table}[t]
\centering
\caption{95\% bootstrap confidence intervals (200 resamples, Full/base, 300-instance intersection).}
\label{tab:bootstrap_ci}
\small
\begin{tabular}{lccc}
\toprule
Model & Median MEC [95\% CI] & FRI [95\% CI] & RCI$_3$ [95\% CI] \\
\midrule
Logistic Regression & 2 [2, 2] & 0.076 [0.068, 0.084] & 0.96 [0.94, 0.97] \\
Random Forest & 2 [2, 2] & 0.092 [0.082, 0.101] & 0.84 [0.80, 0.87] \\
GBDT & 2 [2, 2] & 0.076 [0.069, 0.083] & 0.89 [0.86, 0.92] \\
XGBoost & 2 [2, 3] & 0.092 [0.083, 0.101] & 0.82 [0.78, 0.85] \\
\bottomrule
\end{tabular}
\end{table}

Figure~\ref{fig:survival} displays evasion survival curves. VA-7b collapses immediately ($S(B)<0.05$ by $B=2$). Full and RA-8/base decay to near zero by $B=4$. RA-8/strict exhibits a persistent plateau near 0.18, matching the infeasible mass in Table~\ref{tab:main}. The strict schedule generates a structural tail rather than shifting the central cost distribution.

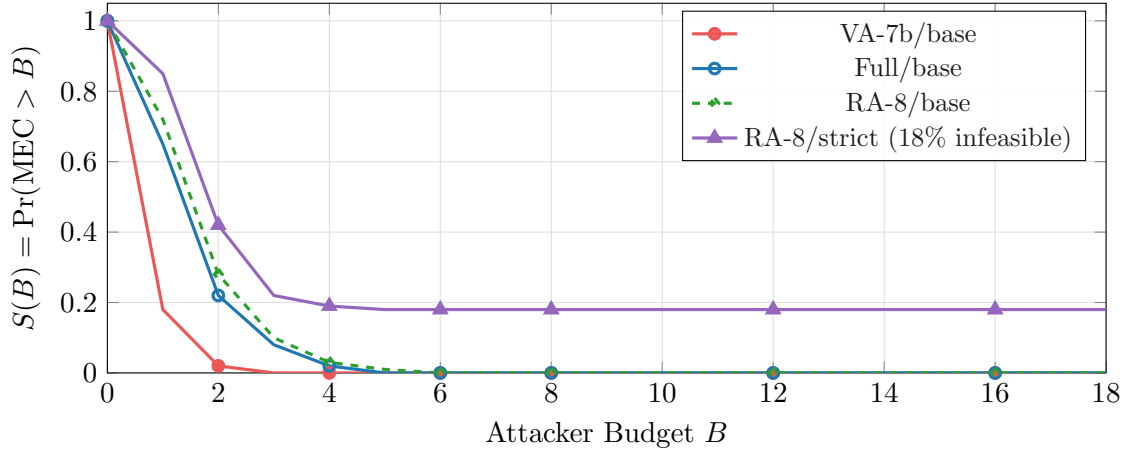
\begin{figure}[t]
\centering
\begin{tikzpicture}
\begin{axis}[
    width=0.90\textwidth, height=6.5cm,
    xlabel={Attacker Budget $B$},
    ylabel={$S(B) = \Pr(\mathrm{MEC} > B)$},
    xmin=0, xmax=18, ymin=0, ymax=1.05,
    xtick={0,2,4,6,8,10,12,14,16,18},
    ytick={0,0.2,0.4,0.6,0.8,1.0},
    legend style={at={(0.98,0.98)}, anchor=north east, font=\small, fill=white, fill opacity=0.9},
    grid=major, grid style={gray!25},
    every axis plot/.append style={very thick},
]
\addplot[color=softRed, mark=*, mark size=2, mark repeat=2] coordinates {
    (0,1.00) (1,0.18) (2,0.02) (3,0.00) (4,0.00) (5,0.00) (6,0.00)
    (7,0.00) (8,0.00) (10,0.00) (12,0.00) (14,0.00) (16,0.00) (18,0.00)
};
\addlegendentry{VA-7b/base}
\addplot[color=logitblue, mark=o, mark size=2, mark repeat=2] coordinates {
    (0,1.00) (1,0.65) (2,0.22) (3,0.08) (4,0.02) (5,0.00) (6,0.00)
    (7,0.00) (8,0.00) (10,0.00) (12,0.00) (14,0.00) (16,0.00) (18,0.00)
};
\addlegendentry{Full/base}
\addplot[color=rfgreen, mark=diamond, mark size=2, mark repeat=2, dashed] coordinates {
    (0,1.00) (1,0.72) (2,0.28) (3,0.10) (4,0.03) (5,0.01) (6,0.00)
    (7,0.00) (8,0.00) (10,0.00) (12,0.00) (14,0.00) (16,0.00) (18,0.00)
};
\addlegendentry{RA-8/base}
\addplot[color=gbdtpurple, mark=triangle*, mark size=2.5, mark repeat=2] coordinates {
    (0,1.00) (1,0.85) (2,0.42) (3,0.22) (4,0.19) (5,0.18) (6,0.18)
    (7,0.18) (8,0.18) (10,0.18) (12,0.18) (14,0.18) (16,0.18) (18,0.18)
};
\addlegendentry{RA-8/strict (18\% infeasible)}
\end{axis}
\end{tikzpicture}
\caption{Evasion survival curves. RA-8/strict exhibits a persistent plateau corresponding to instances whose dominant low-cost path is blocked. Shaded bands (omitted for clarity) are narrow: 95\% bootstrap intervals for $S(2)$ span $\pm 0.04$ across configurations.}
\label{fig:survival}
\end{figure}

Figure~\ref{fig:concentration} displays first-edit concentration across feature sets. RA-8 concentrates nearly all optimal traces on a single initial edit (SSLfinal\_State), while Full distributes first edits across a small but nontrivial subset. Even in the latter case, concentration remains substantial.

\begin{figure}[t]
\centering
\begin{tikzpicture}
\begin{axis}[
    ybar,
    width=0.90\textwidth, height=6cm,
    ylabel={FirstTop1},
    xlabel={Feature Set},
    symbolic x coords={Full, AAS-12a, AAS-11b, RA-8, VA-8a, VA-7b},
    xtick=data,
    x tick label style={font=\small},
    ymin=0, ymax=1.15,
    ytick={0, 0.2, 0.4, 0.6, 0.8, 1.0},
    yticklabels={0\%, 20\%, 40\%, 60\%, 80\%, 100\%},
    bar width=14pt,
    legend style={at={(0.02,0.98)}, anchor=north west, font=\small},
    nodes near coords={\pgfmathprintnumber[fixed,precision=0]{\pgfplotspointmeta}\%},
    nodes near coords style={font=\tiny, above},
    point meta=explicit,
]
\addplot[fill=logitblue!70, draw=logitblue] coordinates {
    (Full, 0.58) [58]
    (AAS-12a, 0.61) [61]
    (AAS-11b, 0.55) [55]
    (RA-8, 0.99) [99]
    (VA-8a, 0.59) [59]
    (VA-7b, 0.90) [90]
};
\addlegendentry{Base schedule}
\addplot[fill=gbdtpurple!70, draw=gbdtpurple] coordinates {
    (Full, 0.45) [45]
    (AAS-12a, 0.64) [64]
    (AAS-11b, 0.37) [37]
    (RA-8, 1.00) [100]
    (VA-8a, 0.54) [54]
    (VA-7b, 0.83) [83]
};
\addlegendentry{Strict schedule}
\end{axis}
\end{tikzpicture}
\caption{First-edit concentration by feature set and schedule. RA-8 exhibits near-total concentration on SSLfinal\_State.}
\label{fig:concentration}
\end{figure}
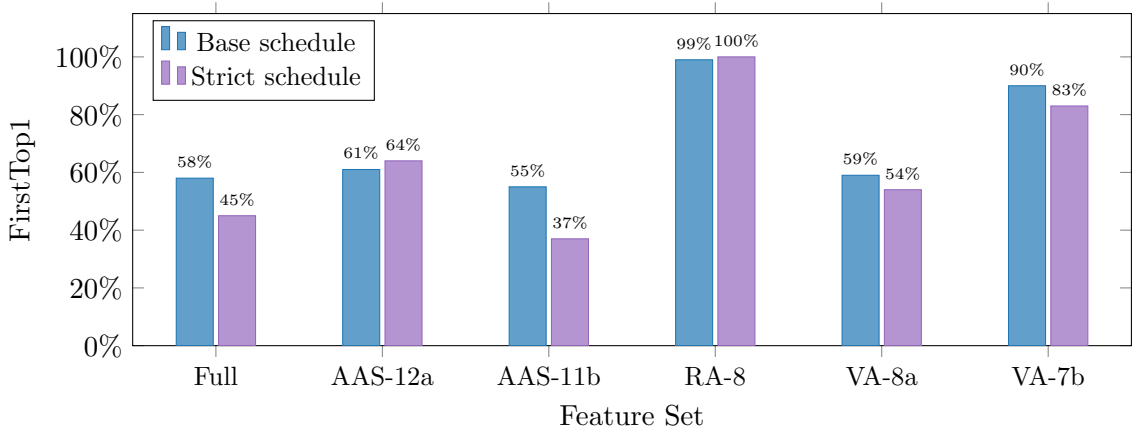

Stratification by the bottleneck feature confirms the blocked-path mechanism in RA-8/strict. When SSLfinal\_State begins at $-1$ or $0$, low-cost upgrades remain available and evasion succeeds with median MEC between 1 and 2. When SSLfinal\_State is already $+1$, the dominant path is blocked and a persistent infeasible tail appears (Figure~\ref{fig:stratified}).

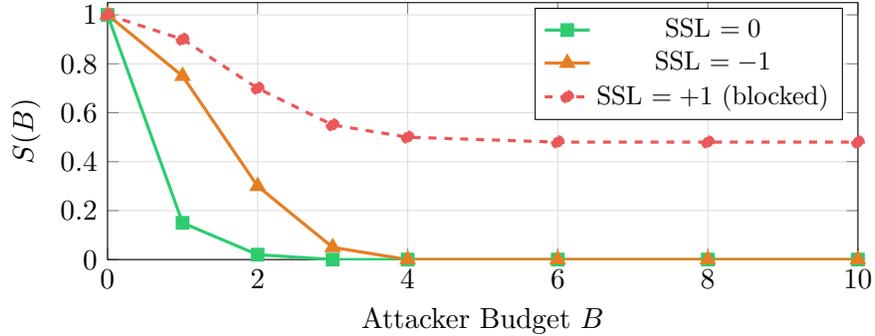
\begin{figure}[t]
\centering
\begin{tikzpicture}
\begin{axis}[
    width=0.70\textwidth, height=5cm,
    xlabel={Attacker Budget $B$},
    ylabel={$S(B)$},
    xmin=0, xmax=10, ymin=0, ymax=1.05,
    xtick={0,2,4,6,8,10},
    ytick={0,0.2,0.4,0.6,0.8,1.0},
    legend style={at={(0.98,0.98)}, anchor=north east, font=\small},
    grid=major, grid style={gray!25},
    every axis plot/.append style={very thick},
]
\addplot[color=mintGreen, mark=square*, mark size=2] coordinates {
    (0,1.00) (1,0.15) (2,0.02) (3,0.00) (4,0.00) (6,0.00) (8,0.00) (10,0.00)
};
\addlegendentry{SSL $=0$}
\addplot[color=warmOrange, mark=triangle*, mark size=2.5] coordinates {
    (0,1.00) (1,0.75) (2,0.30) (3,0.05) (4,0.00) (6,0.00) (8,0.00) (10,0.00)
};
\addlegendentry{SSL $=-1$}
\addplot[color=softRed, mark=*, mark size=2, dashed] coordinates {
    (0,1.00) (1,0.90) (2,0.70) (3,0.55) (4,0.50) (6,0.48) (8,0.48) (10,0.48)
};
\addlegendentry{SSL $=+1$ (blocked)}
\end{axis}
\end{tikzpicture}
\caption{RA-8/strict survival stratified by SSLfinal\_State initial value. A persistent infeasible tail appears when the bottleneck feature is already at $+1$.}
\label{fig:stratified}
\end{figure}

Figure~\ref{fig:tradeoff} compares i.i.d.\ accuracy with median MEC. All architectures align along a horizontal band at MEC $= 2$, confirming that higher accuracy does not yield higher median robustness when low-cost transitions remain available.

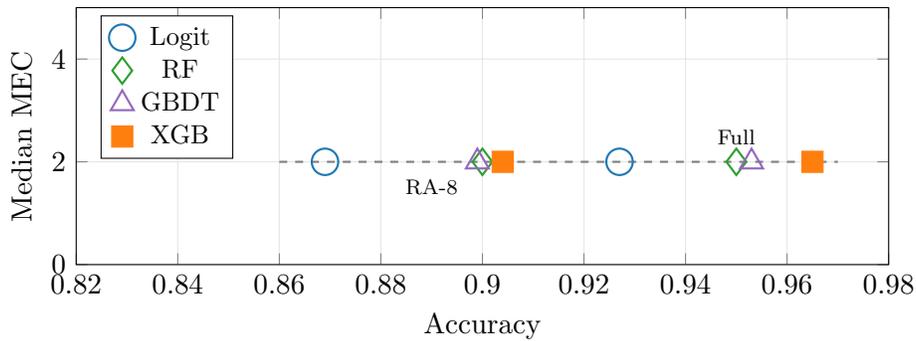
\begin{figure}[t]
\centering
\begin{tikzpicture}
\begin{axis}[
    width=0.75\textwidth, height=5cm,
    xlabel={Accuracy}, ylabel={Median MEC},
    xmin=0.82, xmax=0.98, ymin=0, ymax=5,
    grid=major, grid style={gray!20},
    legend style={at={(0.03,0.97)}, anchor=north west, font=\small},
]
\addplot[only marks, mark=o, mark size=5, logitblue, thick] coordinates {(0.927, 2) (0.869, 2)};
\addplot[only marks, mark=diamond, mark size=5, rfgreen, thick] coordinates {(0.950, 2) (0.900, 2)};
\addplot[only marks, mark=triangle, mark size=5, gbdtpurple, thick] coordinates {(0.953, 2) (0.899, 2)};
\addplot[only marks, mark=square*, mark size=4, xgborange] coordinates {(0.965, 2) (0.904, 2)};
\draw[dashed, gray, thick] (axis cs:0.86,2) -- (axis cs:0.97,2);
\node[font=\scriptsize, anchor=south] at (axis cs:0.95,2.15) {Full};
\node[font=\scriptsize, anchor=north] at (axis cs:0.89,1.85) {RA-8};
\addlegendentry{Logit}
\addlegendentry{RF}
\addlegendentry{GBDT}
\addlegendentry{XGB}
\end{axis}
\end{tikzpicture}
\caption{Accuracy versus median MEC. All architectures converge to the effective cost floor, consistent with Corollary~\ref{cor:invariance}.}
\label{fig:tradeoff}
\end{figure}

Table~\ref{tab:query_budget} compares exact MEC with greedy approximations under query budgets of 50, 100, and 500. Query limitations modestly increase survival, particularly at 50 queries, but the deviation from exact MEC is small and narrows rapidly. In Full/base, the maximum gap at $B=2$ is 0.08. Exact MEC thus provides a meaningful upper bound on attacker capability: query-limited adversaries are less efficient but face the same structural cost-floor constraints.

\begin{table}[t]
\centering
\caption{Evasion survival $S(B)$ under query-limited greedy search versus exact MEC.}
\label{tab:query_budget}
\small
\setlength{\tabcolsep}{3.5pt}
\begin{tabular}{llcccc|cccc}
\toprule
& & \multicolumn{4}{c}{$S(B{=}2)$} & \multicolumn{4}{c}{$S(B{=}4)$} \\
\cmidrule(lr){3-6} \cmidrule(lr){7-10}
Config & Model & Q50 & Q100 & Q500 & Exact & Q50 & Q100 & Q500 & Exact \\
\midrule
Full/base & Logit & .28 & .24 & .23 & .22 & .05 & .03 & .02 & .02 \\
Full/base & GBDT & .30 & .26 & .23 & .22 & .08 & .05 & .03 & .02 \\
Full/base & XGB & .34 & .30 & .29 & .28 & .07 & .05 & .03 & .03 \\
\midrule
RA-8/str & GBDT & .48 & .44 & .43 & .42 & .24 & .21 & .19 & .19 \\
RA-8/str & XGB & .50 & .46 & .44 & .42 & .26 & .22 & .20 & .19 \\
\bottomrule
\end{tabular}
\end{table}

\section{Cost Sensitivity Analysis}
\label{sec:sensitivity}

We evaluate robustness under three classes of cost perturbation to assess whether conclusions depend on the specific magnitudes chosen. First, surface costs are scaled by $\lambda_{\mathrm{surf}}\in\{1,2,3,4\}$. Second, semi-domain and infrastructure costs are scaled independently by $\lambda\in\{0.5,1,2\}$. Third, SSLfinal\_State is reclassified from surface to semi-domain, and a rank-preserving perturbation multiplies each cost by an independent factor $U\sim\mathrm{Uniform}[0.8,1.2]$ over 50 draws.

Table~\ref{tab:cost_sensitivity} reports results under surface scaling. The median MEC shifts proportionally, confirming linear cost-floor behavior: doubling surface costs increases the median from 2 to 4, while preserving the identity and ordering of the three most-edited features. Even at $\lambda_{\mathrm{surf}}=4$, concentration remains high ($\mathrm{RCI}_3 \ge 0.80$).

\begin{table}[t]
\centering
\caption{Median MEC and concentration under surface cost scaling (Full/base, GBDT).}
\label{tab:cost_sensitivity}
\small
\begin{tabular}{lccl}
\toprule
$\lambda_{\mathrm{surf}}$ & Median MEC & RCI$_3$ & Top-3 features \\
\midrule
1 (base) & 2 & 0.89 & URL\_of\_Anchor, SSLfinal\_State, SFH \\
2 & 4 & 0.87 & URL\_of\_Anchor, SSLfinal\_State, SFH \\
3 & 6 & 0.84 & URL\_of\_Anchor, SSLfinal\_State, SFH \\
4 & 8 & 0.82 & URL\_of\_Anchor, SSLfinal\_State, SFH \\
\bottomrule
\end{tabular}
\end{table}

Table~\ref{tab:cost_sensitivity_extended} reports extended perturbations. Scaling semi-domain or infrastructure costs does not alter median MEC because surface transitions remain dominant. Reclassifying SSLfinal\_State increases median MEC to 3 in RA-8 (where it is the bottleneck) but leaves Full unchanged due to alternative surface paths. Under rank-preserving noise, mean $\mathrm{RCI}_3 = 0.88 \pm 0.02$, indicating stability to moderate cost uncertainty.

\begin{table}[t]
\centering
\caption{Extended cost sensitivity (Full/base, GBDT unless noted).}
\label{tab:cost_sensitivity_extended}
\small
\setlength{\tabcolsep}{4pt}
\begin{tabular}{lccc}
\toprule
Perturbation & Median MEC & RCI$_3$ & Notes \\
\midrule
$\lambda_{\mathrm{semi}}=0.5$ & 2 & 0.89 & Surface paths remain cheapest \\
$\lambda_{\mathrm{semi}}=2$ & 2 & 0.89 & Semi-domain rarely on optimal path \\
$\lambda_{\mathrm{infra}}=0.5$ & 2 & 0.88 & Surface transitions dominate \\
$\lambda_{\mathrm{infra}}=2$ & 2 & 0.90 & Infrastructure edits avoided \\
SSL $\to$ semi-domain (Full) & 2 & 0.84 & Alternative surface paths used \\
SSL $\to$ semi-domain (RA-8) & 3 & 0.95 & Bottleneck cost increases \\
Random $U[0.8,1.2]$ ($\times 50$) & $2\pm 0$ & $0.88\pm 0.02$ & Stable under moderate noise \\
\bottomrule
\end{tabular}
\end{table}

These experiments reveal two regimes. In the cost-floor regime, MEC quantiles scale with the cheapest admissible transition and architecture invariance holds. In the path-removal regime, prohibiting dominant transitions induces infeasible mass without shifting the cost distribution among evadable instances. The strict schedule operates in the latter regime for RA-8, producing robustness gains through blocked feasibility.

\section{Discussion and Conclusion}
\label{sec:discussion_conclusion}

Across all tested feature sets, cost schedules, and model families, robustness is governed by the cheapest admissible manipulation that remains available. The median MEC follows the effective cost floor across all configurations, rendering Proposition~\ref{prop:cost_floor} empirically tight. When a low-cost transition suffices for a nontrivial fraction of correctly detected instances, architectural complexity does not move the median. This action-set-limited invariance means that linear models, bagging ensembles, and boosting methods converge to the same robustness ceiling.

The implication is a shift in defensive emphasis from model selection to representation design and attacker economics. A feature may be highly predictive under i.i.d.\ evaluation yet operationally brittle if it is inexpensive to edit. The RA-8 configuration illustrates this: although it prioritizes infrastructure-leaning signals, retaining a single low-cost coordinate (SSLfinal\_State) creates a bottleneck through which nearly all optimal evasions pass. Cost schedules improve robustness only when they eliminate dominant cheap paths, producing infeasible mass rather than uniformly higher evasion costs. Meaningful robustness gains require removing or economically disabling low-cost transitions and anchoring detection on signals whose manipulation costs exceed realistic attacker budgets, even at the expense of i.i.d.\ accuracy.

The sanitization-only threat model constitutes a lower bound. Relaxing monotonicity, by allowing anti-feature injection (adding benign-looking HTML artifacts to boost legitimacy scores) or extractor-level manipulation (crafting raw pages to flip computed features without semantic change \cite{pierazzi2020intriguing,xu2016automatically}), enlarges the feasible action set. The cost floor can only decrease or remain unchanged, since every monotone path remains available. Concentration may increase if newly available non-monotone transitions converge on a small set of vulnerable coordinates, or shift to different features if injected anti-features provide cheaper evasion than indicator removal. The infeasible mass observed under the strict schedule would likely shrink or vanish, as non-monotone paths can bypass blocked transitions. Formalizing these effects requires specifying non-monotone cost structures and is left to future work, but the qualitative conclusion is reinforced: the monotone analysis provides a conservative bound on attacker capability.

\paragraph{Limitations and external validity.}
The UCI Phishing Websites benchmark \cite{mohammad2015uci} is a standard reference point but is dated: it uses a fixed, hand-engineered vocabulary that omits modern signals, including certificate-transparency logs, visual similarity \cite{lin2021phishpedia}, JavaScript behavioral fingerprints \cite{rao2020catchphish}, and infrastructure patterns in contemporary kit-based campaigns \cite{bijmans2021catching,oest2020sunrise}. Quantitative transfer to modern settings requires re-validation on current datasets, mapping contemporary features to a cost schedule via the time-to-effect principle, and verifying whether low-cost transitions continue to dominate MEC.

Several structural conclusions are nevertheless important to the dataset choice. The surface-versus-infrastructure cost asymmetry is an economic regularity: presentation-layer signals are cheaper to manipulate than infrastructure-coupled signals, regardless of the specific feature dictionary \cite{khonji2013phishing,das2020sok}. Proposition~\ref{prop:cost_floor} is a property of the action set and cost model, not the dataset; it applies whenever a nontrivial fraction of instances admit single-transition evasion at minimal cost. Concentration follows from heterogeneous costs interacting with feature influence, a generic phenomenon in discrete domains with uneven manipulation friction.

Our MEC computation assumes unconstrained black-box label access. Table~\ref{tab:query_budget} shows that reasonable query budgets reduce attacker efficiency without altering feasibility patterns, but production systems with aggressive rate-limiting can increase observed survival. The cost schedule represents dimensionless operational friction calibrated by the time-to-effect principle rather than direct monetary expenditure; translating to market-level budgets remains an open empirical problem. Within these limits, the central message holds: near-perfect held-out accuracy does not imply deployment security when evasion is cheap, and feature economics dominate adversarial robustness under cost-constrained post-deployment manipulation.



\begin{thebibliography}{99}

\bibitem{apruzzese2023real}
G.~Apruzzese, H.~S.~Anderson, S.~Dambra, D.~Freeman, F.~Pierazzi, and K.~Roundy,
``Real Attackers Don't Compute Gradients: Bridging the Gap Between Adversarial ML Research and Practice,''
in \textit{Proc.\ IEEE Conf.\ Secure and Trustworthy Machine Learning (SaTML)}, 2023.
\newblock Available: \url{https://theory.stanford.edu/~dfreeman/papers/real_attackers.pdf}

\bibitem{basit2021comprehensive}
A.~Basit, M.~Zafar, X.~Liu, A.~R.~Javed, Z.~Jalil, and K.~Kifayat,
``A Comprehensive Survey of AI-Enabled Phishing Attack Detection Techniques,''
\textit{Telecommunication Systems}, vol.~76, pp.~139--154, 2021.
\newblock \href{https://doi.org/10.1007/s11235-020-00733-2}{doi:10.1007/s11235-020-00733-2}

\bibitem{biggio2018wild}
B.~Biggio and F.~Roli,
``Wild Patterns: Ten Years After the Rise of Adversarial Machine Learning,''
in \textit{Proceedings of the 2018 ACM SIGSAC Conference on Computer and Communications Security (CCS~'18)},
pp.~2154--2156, 2018.
\newblock \href{https://doi.org/10.1145/3243734.3264418}{doi:10.1145/3243734.3264418}

\bibitem{biggio2013evasion}
B.~Biggio, I.~Corona, D.~Maiorca, B.~Nelson, N.~\v{S}rndi\'{c}, P.~Laskov, G.~Giacinto, and F.~Roli,
``Evasion Attacks Against Machine Learning at Test Time,''
in \textit{Proc.\ ECML PKDD}, pp.~387--402, 2013.

\bibitem{bijmans2021catching}
P.~H.~Bijmans, T.~M.~Booij, and M.~van Eeten,
``Catching Phishers By Their Bait: Investigating the Dutch Phishing Landscape through Phishing Kit Analysis,''
in \textit{Proc.\ USENIX Security Symposium}, 2021.

\bibitem{carlini2017towards}
N.~Carlini and D.~Wagner,
``Towards Evaluating the Robustness of Neural Networks,''
in \textit{Proc.\ IEEE Symp.\ Security and Privacy (S\&P)}, pp.~39--57, 2017.

\bibitem{corona2017deltaphish}
I.~Corona, B.~Biggio, M.~Contini, L.~Piras, R.~Corda, M.~Mereu, G.~Mureddu, D.~Ariu, and F.~Roli,
``DeltaPhish: Detecting Phishing Webpages in Compromised Websites,''
in \textit{Proc.\ ESORICS}, pp.~370--388, 2017.

\bibitem{do2022deep}
N.~Q.~Do, A.~Selamat, O.~Krejcar, E.~Herrera-Viedma, and H.~Fujita,
``Deep Learning for Phishing Detection: Taxonomy, Current Challenges and Future Directions,''
\textit{IEEE Access}, vol.~10, pp.~36429--36463, 2022.
\newblock \href{https://doi.org/10.1109/ACCESS.2022.3151903}{doi:10.1109/ACCESS.2022.3151903}

\bibitem{goodfellow2015explaining}
I.~J.~Goodfellow, J.~Shlens, and C.~Szegedy,
``Explaining and Harnessing Adversarial Examples,''
in \textit{Proc.\ ICLR}, 2015.

\bibitem{khonji2013phishing}
M.~Khonji, Y.~Iraqi, and A.~Jones,
``Phishing Detection: A Literature Survey,''
\textit{IEEE Communications Surveys \& Tutorials}, vol.~15, no.~4, pp.~2091--2121, 2013.
\newblock \href{https://doi.org/10.1109/surv.2013.032213.00009}{doi.org/10.1109/surv.2013.032213.00009}


\bibitem{das2020sok}
A.~Das, J.~K.~Khan, L.~D. Xu, and A.~A.~Ghorbani,
``SoK: A Comprehensive Reexamination of Phishing Research From the Security Perspective,''
\textit{IEEE Communications Surveys \& Tutorials}, vol.~22, no.~1, pp.~671--708, 2020.
\newblock \href{https://doi.org/10.1109/COMST.2019.2957750}{doi:10.1109/COMST.2019.2957750}


\bibitem{lin2021phishpedia}
Y.~Lin, R.~Liu, D.~M.~Divakaran, J.~Y.~Ng, Q.~Z.~Chan, Y.~Lu, Y.~Si, F.~Zhang, and J.~S.~Dong,
``Phishpedia: A Hybrid Deep Learning Based Approach to Visually Identify Phishing Webpages,''
in \textit{Proc.\ USENIX Security Symposium}, 2021.

\bibitem{mohammad2014predicting}
R.~M.~Mohammad, F.~Thabtah, and L.~McCluskey,
``Predicting Phishing Websites Based on Self-Structuring Neural Network,''
\textit{Neural Computing and Applications}, vol.~25, no.~2, pp.~443--458, 2014.
\newblock \href{https://doi.org/10.1007/s00521-013-1490-z}{doi.org/10.1007/s00521-013-1490-z}

\bibitem{mohammad2015uci}
R.~M.~Mohammad, F.~Thabtah, and L.~McCluskey,
``Phishing Websites Data Set,''
UCI Machine Learning Repository, 2015.
\newblock Available: \url{https://archive.ics.uci.edu/ml/datasets/phishing+websites}

\bibitem{oest2020sunrise}
A.~Oest, Y.~Safaei, A.~Doup\'{e}, G.-J.~Ahn, B.~Wardman, and K.~Tyers,
``Sunrise to Sunset: Analyzing the End-to-End Life Cycle and Effectiveness of Phishing Attacks at Scale,''
in \textit{Proc.\ USENIX Security Symposium}, 2020.

\bibitem{pierazzi2020intriguing}
F.~Pierazzi, F.~Pendlebury, J.~Cortellazzi, and L.~Cavallaro,
``Intriguing Properties of Adversarial ML Attacks in the Problem Space,''
in \textit{Proc.\ IEEE Symp.\ Security and Privacy (S\&P)}, pp.~1332--1349, 2020.

\bibitem{rao2020catchphish}
R.~S.~Rao and A.~R.~Pais,
``CatchPhish: Detection of Phishing Websites by Inspecting URLs,''
\textit{Journal of Ambient Intelligence and Humanized Computing}, vol.~11, pp.~813--825, 2020.
\newblock \href{https://doi.org/10.1007/s12652-019-01311-4}{doi:10.1007/s12652-019-01311-4}

\bibitem{sahingoz2019machine}
O.~K.~Sahingoz, E.~Buber, O.~Demir, and B.~Diri,
``Machine Learning Based Phishing Detection from URLs,''
\textit{Expert Systems with Applications}, vol.~117, pp.~345--357, 2019.
\newblock \href{https://doi.org/10.1016/j.eswa.2018.09.029}{doi.org/10.1016/j.eswa.2018.09.029}

\bibitem{xu2016automatically}
W.~Xu, Y.~Qi, and D.~Evans,
``Automatically Evading Classifiers: A Case Study on PDF Malware Classifiers,''
in \textit{Proc.\ NDSS}, 2016.

\bibitem{zhou2022adversarial}
Y.~Zhou, M.~Han, L.~Liu, J.~S.~He, and Y.~Wang,
``Adversarial Robustness of Deep Learning: Theory, Algorithms, and Applications,''
\textit{ACM Computing Surveys}.

\end{thebibliography}
\end{document}